 \documentclass[letterpaper, 10 pt, journal,twoside,web]{ieeecolor}
 \usepackage{lcsys}

\usepackage{geometry}
\usepackage{rotating}
\usepackage{chngcntr}
\usepackage{apptools}
\usepackage{listings}
\AtAppendix{\counterwithin{thm}{section}}
\usepackage{graphicx}
\usepackage{graphics}
\usepackage{amssymb}
\usepackage{amsmath}

\usepackage{amsthm}

\usepackage{color}
\usepackage{xspace}
\usepackage{amssymb}
\usepackage{amsfonts}
\usepackage{algpseudocode}
\usepackage{bbm}
\usepackage{comment}
\usepackage{algorithmicx}
\usepackage{subfig}
\usepackage{psfrag}
\usepackage{cite}
\usepackage{algorithm,algorithmicx}
\usepackage[skip=10pt]{subcaption}

\usepackage{xcolor}

\newcommand{\vect}[1]{\boldsymbol{\mathbf{#1}}}

 \newcommand{\boxend}{\hfill \ensuremath{\Box}}

\theoremstyle{definition}

\theoremstyle{remark}
\newtheorem{remark}{Remark}
\newtheorem{thm}{Theorem}[section]

\newtheorem{rem}{Remark}[section]

\newtheorem{lem}{Lemma}[section]

\newtheorem{assump}{Assumption}

\newcommand{\oprocendsymbol}{\hbox{$\bullet$}}
\newcommand{\oprocend}{\relax\ifmmode\else\unskip\hfill\fi\oprocendsymbol}

\makeatletter
\def\@opargbegintheorem#1#2#3{%
  \trivlist
  \item[\hskip\labelsep{\bfseries #1\ #2\ (#3):}]\itshape
}
\makeatother

\begin{document}

\title{ 
\huge{{\huge 
Adaptive Threshold-Driven Continuous Greedy Method for Scalable Submodular Optimization
}}
}

\author{Mohammadreza Rostami, Solmaz S. Kia, \emph{Senior member, IEEE} %
  \thanks{The authors are with the Department of Mechanical and Aerospace Engineering, University of California Irvine, Irvine, CA 92697,  
    {\tt\small \{mrostam2,solmaz\}@uci.edu.} This work was supported by NSF Award ECCS 2452149.}%
}
 
\markboth{}%
{}
\pagenumbering{empty}

\allowdisplaybreaks

\maketitle

\begin{abstract}
Submodular maximization under matroid constraints is a 
fundamental problem in combinatorial optimization with 
applications in sensing, data summarization, active learning, 
and resource allocation. While the Sequential Greedy (SG) 
algorithm achieves only a $\frac{1}{2}$-approximation due to 
irrevocable selections, Continuous Greedy (CG) attains the 
optimal $\bigl(1-\frac{1}{e}\bigr)$-approximation via the 
multilinear relaxation, at the cost of a progressively dense 
decision vector that forces agents to exchange feature 
embeddings for nearly every ground-set element. We propose 
\textit{ATCG} (\underline{A}daptive \underline{T}hresholded 
\underline{C}ontinuous \underline{G}reedy), which gates 
gradient evaluations behind a per-partition progress ratio 
$\eta_i$, expanding each agent's active set only when current 
candidates fail to capture sufficient marginal gain, thereby 
directly bounding which feature embeddings are ever transmitted. 
Theoretical analysis establishes a curvature-aware approximation 
guarantee with effective factor 
$\tau_{\mathrm{eff}}=\max\{\tau,1-c\}$, interpolating between the 
threshold-based guarantee and the low-curvature regime where 
\textit{ATCG} recovers the performance of CG. This shows that the 
problem structure, as captured by curvature, determines the amount of 
coordination and communication required to approach full-CG performance. 
Experiments on a class-balanced prototype selection problem over a subset 
of the CIFAR-10 animal dataset show that \textit{ATCG} achieves objective 
values comparable to those of the full CG method while substantially 
reducing communication overhead through adaptive active-set expansion.
\end{abstract}


\section{Introduction}

Submodular maximization has emerged as a cornerstone of modern discrete 
and combinatorial optimization, with relevance to a wide range of applications 
in machine learning, control, and networked decision-making~\cite{krause2014submodular, 
nemhauser1978analysis, calinescu2011maximizing, bach2013learning, rostami2023federated, rostami2024fedscalar, rostami2025fedmpdd}. 
Submodular functions naturally model notions of diversity, 
representativeness, and information coverage, and thus appear in problems 
such as sensor placement~\cite{krause2008near}, data 
summarization~\cite{lin2011class}, influence maximization~\cite{kempe2003maximizing}, 
and active learning~\cite{golovin2011adaptive}.

\begin{figure}[t]
\centering
    \includegraphics[scale=0.23]{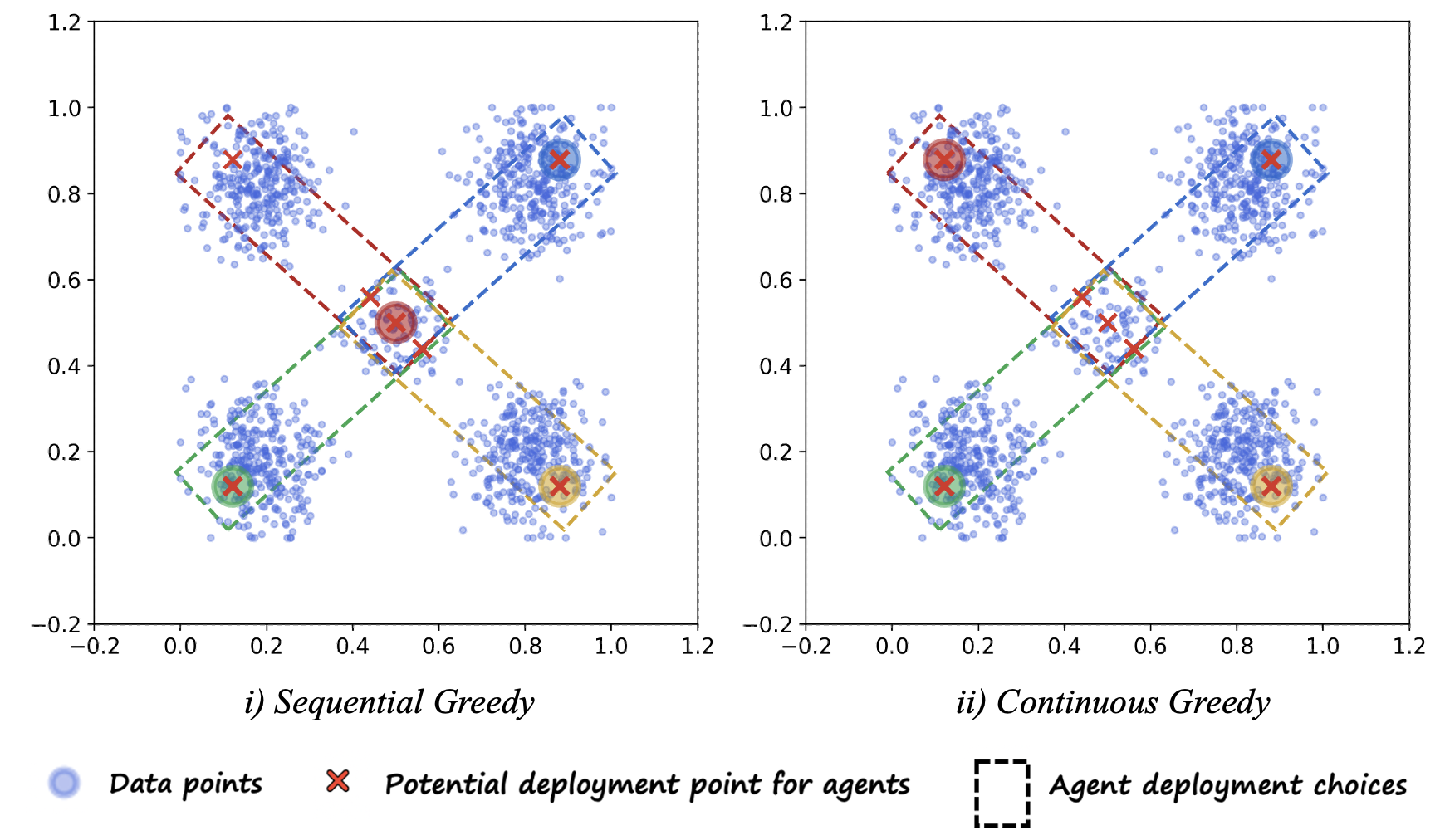}
\caption{{\small 
Illustration of the early-commitment effect of SG 
(left) and its mitigation via CG (right) in a 2D sensor placement 
task. Deployment points (solid circles) are selected from candidates 
(red~$\times$) to cover a clustered data distribution (blue dots) 
under partition-matroid constraints, where each colored region 
represents the candidate set of the corresponding agent.}}
    \label{fig:forward_gradiant}
\end{figure}

A central problem in this area is constrained submodular maximization 
under a \emph{matroid constraint}, which generalizes independence 
structures such as cardinality, partition, or budget 
limits~\cite{calinescu2011maximizing}. The canonical formulation with 
utility function $f:2^\mathcal{P}\to\mathbb{R}_{+}$ is
\begin{equation}\label{eq::canonical_prob}
\max_{\mathcal{S} \in \mathcal{I}} f(\mathcal{S}),
\end{equation}
where $\mathcal{I}$ denotes the independent sets of a matroid on ground 
set $\mathcal{P}$. This problem is NP-hard. 

The celebrated \emph{Sequential Greedy (SG)} algorithm~\cite{nemhauser1978analysis, 
fisher2009analysis} offers a classical polynomial-time solution by 
iteratively selecting the element with the largest marginal gain in 
$f(\mathcal{S})$, starting from the empty set. Despite its simplicity 
and scalability, SG achieves only a $(1/2)$-approximation for monotone 
submodular functions under general matroid constraints. This limitation 
arises from the inherently \emph{irreversible} nature of greedy 
selections, i.e., once an element is chosen, the decision cannot be revised, 
which can lead to suboptimal configurations when early selections 
restrict future possibilities. The \emph{Continuous Greedy (CG)} 
algorithm~\cite{calinescu2011maximizing} closes this gap by optimizing 
over the \emph{multilinear extension} of $f$,
\begin{align}\label{eq::multi_linear}
F(\vect{x}) =& \mathbb{E}[f(\mathcal{R}(\vect{x}))]=\sum_{\mathcal{R} \subset \mathcal{P}}{}\! \!f(\mathcal{R})\! \prod_{p \in \mathcal{R}}^{}\!\! [\vect{x}]_p \prod_{p \not\in \mathcal{R}}^{}\!(1-[\vect{x}]_p),
\end{align}
where $\vect{x}\in[0,1]^{|\mathcal{P}|}$ is the \emph{probability membership vector}, and $\mathcal{R}(\mathbf{x})$ includes each 
element $p$ independently with probability $[\vect{x}]_p\in[0,1]$. Instead of committing 
to discrete selections, CG maintains a fractional vector $\vect{x}$ that 
evolves continuously within the matroid polytope, at each step following 
a direction that maximizes the inner product with the gradient 
$\nabla F(\vect{x})$ and gradually increasing the probability mass assigned to 
promising elements. By integrating these infinitesimal greedy steps, CG 
achieves the optimal $(1-1/e)$ approximation, with the final solution 
recovered via a lossless rounding step. 

The qualitative difference between SG and CG is illustrated 
in Fig.~\ref{fig:forward_gradiant}: SG's early commitment 
to a central median point fails to capture the four distinct 
data clusters, while CG's deferred, fractional allocation 
produces a superior arrangement. Tightening the optimality 
gap directly translates to cost savings and efficiency gains 
across domains such as sensor scheduling, logistics, and 
robotic coverage.

Despite its theoretical elegance and improved approximation guarantee, 
deploying CG in distributed or networked systems introduces significant 
computational and communication challenges. Each iteration requires agents 
evaluating or estimating the gradient $\nabla F(\vect{x})$, which depends on 
the contribution of all elements of the ground set $\mathcal{P}$---distributed 
among agents across the network, as in the partition matroid case,~see~\cite{AM-HH-AK:18, rezazadeh2023distributed}. 
Critically, since CG operates over the multilinear extension, the 
probability membership vector $\vect{x}$ evolves continuously and its components gradually 
become non-zero over the course of the algorithm. Each nonzero component $[\vect{x}(t)]_j>0$ indicates that element $j$ is 
active in the gradient estimation, so every agent must access its feature 
embedding, i.e., the vector representation used to evaluate marginal gains 
and compute $\nabla F(\vect{x})$. As $\vect{x}$ tends to become dense throughout the CG trajectory, 
this effectively forces agents to exchange the feature data of nearly 
every element in the ground set which is resulting in a communication burden 
that scales with $|\mathcal{P}|$ rather than the size of any local 
selection. This stands in sharp contrast to SG, where each agent need 
only broadcast its single final selected element upon termination, 
incurring a one-time, minimal data transmission at the price of a 
suboptimal $(1/2)$-approximation~guarantee.

\emph{Statement of contribution}: To address this overhead, we propose 
\textit{ATCG} (\underline{A}daptive \underline{T}hresholded 
\underline{C}ontinuous \underline{G}reedy), a threshold-based variant 
of CG demonstrated in a server-assisted distributed setting under a 
partition-matroid structure, where each agent maintains a local partition of the ground set and local partition 
of the probability membership vector $\vect{x}$ and relies on a central server for data exchange. Each agent selectively restricts gradient evaluations to 
a small, dynamically expanding active subset, triggering expansion 
only when the ratio of the best marginal gain within the active subset 
to that of the full partition falls below a predefined threshold 
$\tau$---the ``Adaptive'' in the name reflects this 
iteration-by-iteration expansion rule, in contrast to static or 
pre-defined truncation schemes. Since only active elements contribute 
non-zero entries to $\vect{x}$, this directly limits the volume of feature 
data exchanged between agents and the server, keeping communication 
overhead proportional to the active set size rather than 
$|\mathcal{P}|$. Theoretical analysis establishes that the coverage 
factor $\tau_{\mathrm{eff}} = \max\{\tau, 1-c\}$, jointly determined by the 
threshold parameter and the function curvature, interpolates between a threshold-controlled guarantee and a low-curvature regime where \textit{ATCG} approaches the performance of CG. A numerical example demonstrates the efficiency of the proposed algorithm in reducing communication to the server.

\emph{Related work}: Thresholding and hard-threshold operators have a 
rich history in optimization and signal processing. The compressed 
sensing literature has developed \emph{iterative hard thresholding} 
(IHT) and its variants to enforce sparsity constraints~\cite{pan2017convergent, 
yuan2014gradient, foucart2011hard}, and related works apply constrained 
thresholded updates in nonconvex optimization~\cite{liu2020between}. 
While superficially similar in name, our mechanism differs fundamentally 
in both goal and operation. The work in \cite{pan2017convergent} targets 
sparsity-constrained minimization via Armijo-type step sizes; 
\cite{yuan2014gradient} enforces sparsity per iteration via thresholded 
projection onto candidate support sets; \cite{foucart2011hard} selects 
signal support under linear measurement models; and \cite{liu2020between} 
alternates thresholding with gradient descent to balance convergence 
stability and sparsity. In contrast, \textit{ATCG} uses 
thresholding not to enforce sparsity, but to \emph{adaptively decide 
when to expand} the active set for gradient evaluation which is guided by a 
ratio of marginal gains, operating prtition-wise under a partition matroid, 
and motivated by communication efficiency rather than sparse 
minimization.

Scalability in submodular optimization has also been pursued through 
two complementary directions. \emph{Lazy greedy}~\cite{minoux1978accelerated},\cite{mirzasoleiman2015lazier} accelerates SG by caching marginal gain upper bounds 
to skip redundant evaluations, but inherits the $(1/2)$-approximation 
ceiling of SG and does not extend to the continuous domain. 
MapReduce-based approaches~\cite{mirzasoleiman2013distributed, 
kumar2015fast} achieve one-round communication efficiency through a 
partition-and-merge paradigm, but incur approximation loss---typically 
$(1/4)$ or worse---due to the absence of global coordination during 
local selection. \textit{ATCG} operates entirely within the 
CG framework, preserving the $(1-1/e)$ guarantee while achieving 
communication efficiency through an adaptive $\tau$-threshold rule, 
without partitioning computation or sacrificing near-optimality.

\section{Problem Definition}\label{sec::problem}

Consider optimization problem~\eqref{eq::canonical_prob} where the 
ground set $\mathcal{P} = \bigcup_{i=1}^{N}\mathcal{P}_i$ is a finite 
set, consisted of partitioned  $N$ disjoint subsets, where partition $\mathcal{P}_i$ 
represents the candidate elements associated with agent $i$. The utility 
function $f:2^{\mathcal{P}}\to\mathbb{R}_{+}$ is \emph{submodular}, 
i.e., it satisfies the diminishing-returns property:
\begin{equation*}
f(\mathcal{A}\cup\{e\})-f(\mathcal{A})\geq 
f(\mathcal{B}\cup\{e\})-f(\mathcal{B}),
~~ \forall\mathcal{A}\subseteq\mathcal{B}\subseteq\mathcal{P},
\; e\notin\mathcal{B}.
\end{equation*}
Moreover, it is \emph{monotone}, i.e., $f(\mathcal{A})\leq f(\mathcal{B})$ 
whenever $\mathcal{A}\subseteq\mathcal{B}$ and normal, i.e., $f(\emptyset)=0$. The matroid constraint is 
a \emph{partition matroid} 
that enforces at most $\kappa_i$ elements 
may be selected from each partition, defining the family of feasible sets:
\begin{equation}\label{eq::partion_matroid}
\mathcal{I}=\big\{\mathcal{S}\subseteq\mathcal{P}\;\big|\;
|\mathcal{S}\cap\mathcal{P}_i|\leq \kappa_i,\;\forall\,i\in\{1,\cdots,N\}\big\}.
\end{equation}
Following~\cite{calinescu2011maximizing}, to solve~\eqref{eq::canonical_prob} 
tractably, we optimize the multilinear extension $F(\mathbf{x})$ 
defined in~\eqref{eq::multi_linear} over the \emph{matroid polytope}
\begin{equation}\label{eq::polytope}
\mathcal{M}=\Big\{\mathbf{x}\in[0,1]^{|\mathcal{P}|}\,\Big|\,
\textstyle\sum\limits_{j\in\mathcal{P}_i}[\mathbf{x}]_j
\leq \kappa_i,\;\forall\,i\in\{1,\cdots,N\}\Big\},
\end{equation}
the convex hull of the feasible indicator vectors, yielding the 
relaxed problem
\begin{equation}\label{eq::cont}
\max_{\mathbf{x}\in\mathcal{M}}\;F(\mathbf{x}).
\end{equation}
The CG algorithm~\cite{calinescu2011maximizing} solves~\eqref{eq::cont} 
by initializing $\mathbf{x}(0)=\mathbf{0}$ and following the 
continuous-time ascent flow for $t\in[0,1]$
\begin{equation}\label{eq::CG_update}
\mathbf{v}(t)=\operatorname*{arg\,max}_{\mathbf{v}\in\mathcal{M}}\;
\mathbf{v}^{\!\top}\nabla F(\mathbf{x}(t)),
\qquad
\frac{\mathrm{d}}{\mathrm{d}t}\mathbf{x}(t)=\mathbf{v}(t),
\end{equation}
where the gradient of $F$ admits the stochastic representation
\begin{equation}\label{eq::grad}
[\nabla F(\mathbf{x})]_j = 
\mathbb{E}\big[f(\mathcal{R}(\mathbf{x})\cup\{j\})
-f(\mathcal{R}(\mathbf{x})\setminus\{j\})\big].
\end{equation}
A lossless rounding step using $\vect{x}(1)$ then recovers a feasible 
$\mathcal{S}\in\mathcal{I}$, achieving the 
$(1-1/e)$-approximation guarantee~\cite{calinescu2011maximizing}. 
In practice, the continuous flow~\eqref{eq::CG_update} is 
discretized into $T$ steps
\begin{equation}\label{eq::CG_disc}
\mathbf{x}\!\left(t+\tfrac{1}{T}\right)=
\mathbf{x}(t)+\tfrac{1}{T}\,\mathbf{v}(t),
\end{equation}
with the gradient estimated via Monte Carlo 
sampling~\cite{calinescu2011maximizing}.

Without loss of generality, we let $\mathcal{P}=\{1,\ldots,n\}$ 
with $n\triangleq|\mathcal{P}|$, where the elements of each partition 
$\mathcal{P}_i$ occupy a contiguous range of indices. Under this 
labeling, the $j$-th component of the probability membership  vector and its 
corresponding gradient entry are written as $[\mathbf{x}]_j$ and 
$[\nabla F(\mathbf{x})]_j$ for $j\in\mathcal{P}_i$, with the owning 
partition $i$ implicitly determined by the index $j$. The non-negativity of the multilinear gradient for monotone submodular functions, 
combined with the partition-wise structure of the matroid polytope~\eqref{eq::polytope}, 
ensures that the linear oracle in~\eqref{eq::CG_update} decomposes partition-wise. 
For $\kappa_i=1$,\footnote{While we present the case for $\kappa_i=1$ for clarity and 
ease of notation, our results and the ATCG algorithm apply to 
$\kappa_i > 1$; we refer the reader to Algorithm \ref{alg:atcg:kappa} in the Appendix for general $\kappa_i$.} the oracle to choose $\mathbf{v}(t)$ reduces to identifying, 
within each partition $i$, the single element with the largest gradient value:
\begin{equation}\label{eq::oracle_decomp}
    j_i^\star = \operatorname*{arg\,max}_{j \in \mathcal{P}_i}
    [\nabla F(\mathbf{x}(t))]_j,
\end{equation}
and the optimal ascent direction $\mathbf{v}^\star(t) \in \mathcal{M}$ is
defined componentwise as
\begin{equation}\label{eq::ascent_decomp}
    [\mathbf{v}^\star(t)]_j =
    \begin{cases}
        1 & \text{if } j = j_i^\star \text{ for some } i \in [N], \\
        0 & \text{otherwise,}
    \end{cases}
\end{equation}
so that exactly one entry of $\mathbf{x}$ per partition is
incremented by $\tfrac{1}{T}$ at each step, with the
updated entry identified by the global index $j_i^\star \in
\mathcal{P}_i$. This partitionwise separability makes CG
naturally amenable to a server-assisted distributed
realization.

\subsection{Server-Assisted Distributed Realization}
We consider a server-assisted architecture\footnote{The
server-assisted architecture bears structural resemblance to
federated learning: agents retain local data ownership,
communicate only task-relevant updates, and coordinate
through a central aggregator. However, unlike
FedAvg~\cite{mcmahan2017communication}---where the server
passively averages local model updates---the server here
serves as a coordination and assembly point, maintaining
the global decision vector $\mathbf{x}$ and the active
embedding set $\mathcal{E}$, while gradient estimation is
performed locally and in parallel by each agent.} in which
$N$ agents, each owning partition $\mathcal{P}_i$,
collaborate through a central server to execute CG in
parallel. Each agent $i$ initializes and maintains its own
local sub-vector $\mathbf{x}_i(0)=\mathbf{0}$ of the
global decision vector
$\mathbf{x}=[\mathbf{x}_1^\top,\ldots,\mathbf{x}_N^\top]^\top$.
Each iteration of the distributed CG is formalized in
Algorithm~\ref{alg:cg}: the server broadcasts the current
active embedding set $\mathcal{E}$ and decision vector
$\mathbf{x}$ to each agent (transmitting only the
components updated since the previous iteration); each
agent independently estimates its block gradient
$\{[\nabla F(\mathbf{x})]_j\}_{j\in\mathcal{P}_i}$ via
Monte Carlo sampling over $\mathcal{E}$, solves its local
oracle~\eqref{eq::oracle_decomp}, and updates its
sub-vector; agents upload their updated sub-vectors
together with the feature embedding of any newly activated
element; and the server reassembles the global state for
the next iteration.

This architecture realizes the centralized CG algorithm through
fully parallel local gradient computation at each agent similar to distributed solutions in~\cite{AM-HH-AK:18, rezazadeh2023distributed}, with the
server responsible for state assembly and incremental broadcasting.
The communication overhead is governed by the number of elements
that ever become active: each time $[\mathbf{x}_i]_{j}$ transitions
from zero to non-zero, the feature embedding of element $j$ must be
transmitted to the server and added to $\mathcal{E}$, so that it can
be broadcast to all agents for inclusion in subsequent Monte Carlo
gradient estimates via~\eqref{eq::grad}. Since each agent's local
gradient estimation requires evaluating
$f(\mathcal{R}(\mathbf{x})\cup\{j\})$ over all active elements in
$\mathcal{R}(\mathbf{x})$, every newly activated element enlarges
the sampling workload of every agent. As CG progresses and
$\mathbf{x}$ becomes dense, this forces the transmission and
caching of feature data for nearly every element of $\mathcal{P}$,
resulting in communication overhead that scales with
$|\mathcal{P}|$ rather than the $N$ locally selected elements.
Controlling which elements ever become active---and therefore which
feature embeddings are ever uploaded and broadcast---is therefore
the central motivation for \textit{ATCG}.

\begin{algorithm}[t]\footnotesize
\caption{\textbf{CG} in Server-Assisted Realization}
\label{alg:cg}
\begin{algorithmic}[1]\footnotesize
\Require Ground set $\mathcal{P}=\bigcup_{i=1}^N\mathcal{P}_i$
with disjoint partitions $\mathcal{P}_i$;
horizon $T$; number of MC samples $K$
\State \textbf{[Server]}\;
$\mathbf{x}\leftarrow\mathbf{0}$;\;
$\mathcal{E}\leftarrow\emptyset$
\State \textbf{[Agent $i$, $\forall i$]}\;
$\mathbf{x}_i\leftarrow\mathbf{0}$
\For{$t=0,1,\ldots,T-1$}
    \State \textbf{[Server$\,\to\,$Agent $i$]}\;
    Broadcast $\mathcal{E}$ and $\mathbf{x}$ to each agent
    \hfill$\triangleright$ \textit{Incremental broadcast}
    \For{$i=1,\ldots,N$}
    \hfill$\triangleright$ \textit{Local computation (parallel)}
        \State \textbf{[Agent $i$]}\;
        Estimate $[\mathbf{g}]_j \approx [\nabla F(\mathbf{x})]_j$
        for all $j\in\mathcal{P}_i$ using $K$ MC samples
        via~\eqref{eq::grad} over $\mathcal{E}$
        \hfill$\triangleright$ \textit{Gradient estimation}
        \State \textbf{[Agent $i$]}\;
        $j_i^\star \leftarrow
        \operatorname*{arg\,max}_{j\in\mathcal{P}_i}[\mathbf{g}]_j$
        \hfill $\triangleright$ via~\eqref{eq::oracle_decomp}
        \State \textbf{[Agent $i$]}\;
        $[\mathbf{x}_i]_{j_i^\star} \leftarrow
        [\mathbf{x}_i]_{j_i^\star} + \tfrac{1}{T}$
    \EndFor
    \For{$i=1,\ldots,N$}
    \hfill$\triangleright$ \textit{Upload (parallel)}
        \State \textbf{[Agent $i\,\to\,$Server]}\;
        Transmit $\mathbf{x}_i\!\left(t+\tfrac{1}{T}\right)$
        \If{$[\mathbf{x}_i]_{j_i^\star}$ becomes non-zero
        for the first time}
            \State \textbf{[Agent $i\,\to\,$Server]}\;
            Transmit embedding of $j_i^\star$;\;
            $\mathcal{E}\leftarrow\mathcal{E}\cup\{j_i^\star\}$
            \hfill$\triangleright$ \textit{Embedding upload}
        \EndIf
    \EndFor
    \State \textbf{[Server]}\; Assemble
    $\mathbf{x}\!\left(t+\tfrac{1}{T}\right)=
    [\mathbf{x}_1^\top,\ldots,\mathbf{x}_N^\top]^\top$
\EndFor
\State \textbf{[Server]}\;
\textbf{Round} $\mathbf{x}(1)$ to feasible $\mathcal{S}\in\mathcal{I}$
\State \Return $\mathcal{S}$
\end{algorithmic}
\end{algorithm}

\section{Proposed Algorithm and Theoretical 
Guarantees}\label{sec:algo_theory}

\textit{ATCG} (Algorithm~\ref{alg:atcg}) modifies the CG 
oracle by restricting gradient evaluations to dynamically 
expanding \emph{active sets} $\{\mathcal{A}_i\}_{i=1}^N$, 
where $\mathcal{A}_i \subseteq \mathcal{P}_i$ contains the 
candidate elements of partition $i$ currently participating in 
the greedy updates. Elements are admitted to $\mathcal{A}_i$ 
only when the current active set is no longer sufficiently 
representative of the best available marginal gain in 
$\mathcal{P}_i$. Since $[\mathbf{x}]_j$ can become 
non-zero only if $j \in \mathcal{A}_i$ at some iteration, 
this selective activation directly bounds which feature 
embeddings must ever be transmitted to the server.

\paragraph{Progress ratio.}
At each iteration, the quality of the active set 
$\mathcal{A}_i$ is measured by the \emph{progress ratio}
\begin{equation}\label{eq:eta_def}
\eta_i =
\frac{\max_{j\in\mathcal{A}_i}
[\nabla F(\mathbf{x})]_j}
{\max_{j\in\mathcal{P}_i}
[\nabla F(\mathbf{x})]_j},
\end{equation}
the ratio of the best marginal gain within the active 
set to that of the full partition $\mathcal{P}_i$. When 
$\eta_i\geq\tau$, the active set captures at least 
a $\tau$-fraction of the maximum available marginal 
gain and no expansion is needed. When $\eta_i<\tau$ 
and provided $\mathcal{A}_i\neq\mathcal{P}_i$, \textit{ATCG} 
admits the element with the largest gradient value 
outside the current active set:
\begin{equation*}
j_i^\star \in 
\operatorname*{arg\,max}_{j\in\mathcal{P}_i
\setminus\mathcal{A}_i}
[\nabla F(\mathbf{x})]_j,
\qquad
\mathcal{A}_i \leftarrow \mathcal{A}_i\cup\{j_i^\star\}.
\end{equation*}
The term ``Adaptive'' in \textit{ATCG} reflects this 
iteration-by-iteration expansion rule: the active set 
grows only when the current set is demonstrably 
insufficient, in contrast to static or pre-defined 
truncation schemes that fix participation in advance.

\paragraph{Ascent direction and update.}
Given $\{\mathcal{A}_i\}_{i=1}^N$, \textit{ATCG} 
restricts the oracle~\eqref{eq::oracle_decomp} to the 
active set, selecting within each partition the active 
element with the largest gradient value:
\begin{equation}\label{eq:atcg_dir}
j_i^\star \in \operatorname*{arg\,max}_{j\in\mathcal{A}_i}
[\nabla F(\mathbf{x})]_j,
\qquad
[\mathbf{v}]_{j_i^\star}=1,
\end{equation}
with $[\mathbf{v}]_j=0$ otherwise, and the decision 
vector updated as 
$\mathbf{x}\leftarrow\mathbf{x}+\tfrac{1}{T}\mathbf{v}$. 
After $T$ iterations, pipage or swap rounding converts 
the fractional solution $\mathbf{x}(1)\in\mathcal{M}$ 
into a feasible discrete set $\mathcal{S}\in\mathcal{I}$.

\paragraph{Distributed realization.}
\textit{ATCG} admits an exact realization under the 
server-assisted protocol of Section~\ref{sec::problem}, 
formalized in Algorithm~\ref{alg:atcg}. At each 
iteration, upon receiving the global state $\mathbf{x}$ 
and the current embedding set $\mathcal{E}$ from the server, 
each agent $i$ independently estimates its local block 
gradient $\{[\mathbf{g}]_j\}_{j\in\mathcal{P}_i}$ 
via~\eqref{eq::grad}. The agent then evaluates 
the progress ratio~\eqref{eq:eta_def} and, if 
$\eta_i < \tau$, expands its active set $\mathcal{A}_i$ 
by adding the best inactive element. Crucially, the 
feature embedding of this new element is transmitted to 
the server only upon its initial entry into $\mathcal{A}_i$, 
marking the first time its corresponding coordinate in 
$\mathbf{x}_i$ can become non-zero. Agent $i$ then 
computes its local ascent direction via~\eqref{eq:atcg_dir}, 
updates its sub-vector $\mathbf{x}_i$, and uploads it to the 
server. All per-partition computations proceed in 
parallel, with the server as the sole 
coordination and assembly~point.

\paragraph{Communication efficiency.}
Since $[\mathbf{x}]_j$ becomes non-zero only upon 
activation into $\mathcal{A}_i$, the total number of 
distinct feature embeddings ever uploaded to the server 
is bounded by 
$\sum_{i=1}^N|\mathcal{A}_i|\ll|\mathcal{P}|$ for 
moderate $\tau$. By triggering expansion only when 
$\eta_i<\tau$, \textit{ATCG} prevents unnecessary feature 
transmissions whenever the active set already captures 
near-optimal marginal gain, keeping communication cost 
proportional to $|\mathcal{A}_i|$ rather than 
$|\mathcal{P}_i|$.

\begin{algorithm}[t]\footnotesize
\caption{\textbf{\textit{ATCG}} in Server-Assisted Realization}
\label{alg:atcg}
\begin{algorithmic}[1]\footnotesize
\Require Ground set $\mathcal{P}=\bigcup_{i=1}^N\mathcal{P}_i$
with disjoint partitions $\mathcal{P}_i$; threshold $\tau>0$;
horizon $T$; number of MC samples $K$
\State \textbf{[Server]}\;
$\mathbf{x}\leftarrow\mathbf{0}$;\;
$\mathcal{E}\leftarrow\emptyset$
\State \textbf{[Agent $i$, $\forall i$]}\;
$\mathbf{x}_i\leftarrow\mathbf{0}$;\;
$\mathcal{A}_i\leftarrow\emptyset$
\For{$t=0,1,\ldots,T-1$}
    \State \textbf{[Server$\,\to\,$Agent $i$]}\;
    Broadcast $\mathcal{E}$ and $\mathbf{x}$ to each agent
    \hfill$\triangleright$ \textit{Incremental broadcast}
    \For{$i=1,\ldots,N$}
    \hfill$\triangleright$ \textit{Local computation (parallel)}
        \State \textbf{[Agent $i$]}\;
        Estimate $[\mathbf{g}]_j \approx [\nabla F(\mathbf{x})]_j$
        for all $j\in\mathcal{P}_i$ using $K$ MC samples
        via~\eqref{eq::grad} over $\mathcal{E}$
        \hfill$\triangleright$ \textit{Gradient estimation}
        \State \textbf{[Agent $i$]}\;
        $\eta_i \leftarrow 0$ if $\mathcal{A}_i=\emptyset$,\;
        else\;
        $\eta_i\leftarrow
        \dfrac{\max_{j\in\mathcal{A}_i}[\mathbf{g}]_j}
        {\max_{j\in\mathcal{P}_i}[\mathbf{g}]_j+10^{-12}}$
        \hfill$\triangleright$ \textit{Progress ratio}~\eqref{eq:eta_def}
        \If{$\eta_i < \tau$ \textbf{and}
        $\mathcal{A}_i \neq \mathcal{P}_i$}
        \hfill$\triangleright$ \textit{Active-set expansion}
            \State \textbf{[Agent $i$]}\;
            $j_i^\star \leftarrow
            \operatorname*{arg\,max}_{j\in
            \mathcal{P}_i\setminus\mathcal{A}_i}
            [\mathbf{g}]_j$;\;
            $\mathcal{A}_i \leftarrow
            \mathcal{A}_i \cup \{j_i^\star\}$
            \State \textbf{[Agent $i\,\to\,$Server]}\;
            Transmit embedding of $j_i^\star$;\;
            $\mathcal{E}\leftarrow\mathcal{E}
            \cup\{j_i^\star\}$
            \hfill$\triangleright$ \textit{Embedding upload}
        \EndIf
        \If{$\mathcal{A}_i \neq \emptyset$}
        \hfill$\triangleright$ \textit{Active-set oracle}~\eqref{eq:atcg_dir}
            \State \textbf{[Agent $i$]}\;
            $j_i^\star \leftarrow
            \operatorname*{arg\,max}_{j\in\mathcal{A}_i}
            [\mathbf{g}]_j$
            \State \textbf{[Agent $i$]}\;
            $[\mathbf{x}_i]_{j_i^\star} \leftarrow
            [\mathbf{x}_i]_{j_i^\star} + \tfrac{1}{T}$
        \EndIf
    \EndFor
    \For{$i=1,\ldots,N$}
    \hfill$\triangleright$ \textit{Upload (parallel)}
        \State \textbf{[Agent $i\,\to\,$Server]}\;
        Transmit $\mathbf{x}_i\!\left(t+\tfrac{1}{T}\right)$
    \EndFor
    \State \textbf{[Server]}\; Assemble
    $\mathbf{x}\!\left(t+\tfrac{1}{T}\right)=
    [\mathbf{x}_1^\top,\ldots,\mathbf{x}_N^\top]^\top$
\EndFor
\State \textbf{[Server]}\;
\textbf{Round} $\mathbf{x}(1)$ to feasible $\mathcal{S}\in\mathcal{I}$
\State \Return $\mathcal{S}$
\end{algorithmic}
\end{algorithm}

\subsection{Performance Guarantee of \textit{ATCG}}

\begin{assump}[Exact-gradient and continuous-time idealization]
\label{rem:exact_grad}
The guarantees below are stated for the exact gradient
$\nabla F(\mathbf{x})$ and the continuous-time ascent
flow~\eqref{eq::CG_update} where $\vect{v}(t)$ is decided
by \textit{ATCG} via~\eqref{eq:atcg_dir}.\hfill$\square$
\end{assump}

The impact of the practical discretization~\eqref{eq::CG_disc}
and Monte Carlo approximation~\eqref{eq::grad} follows the
standard analysis of~\cite{calinescu2011maximizing} and is
omitted; the finite-step error vanishes at $\mathcal{O}(1/T)$
and the gradient-estimation error decays at
$\mathcal{O}(1/\sqrt{K})$.

\begin{thm}[Performance guarantee of \textit{ATCG}, Algorithm~\ref{alg:atcg}]
\label{thm:atcg_tau}
Let \(f:2^{\mathcal{P}}\to\mathbb{R}_{+}\) be a monotone submodular
function, let \(F:[0,1]^{|\mathcal{P}|}\to\mathbb{R}_{+}\) be its
multilinear extension, and let \(\mathcal{M}\) denote the matroid
polytope associated with~\eqref{eq::canonical_prob}. Consider the
continuous-time \textit{ATCG} trajectory \(\mathbf{x}(t)\in\mathcal{M}\),
\(t\in[0,1]\). For some \(\tau\in(0,1]\), the active sets
satisfy
\begin{equation}\label{eq:tau-coverage}
\max_{j\in\mathcal{A}_i(t)}[\nabla F(\mathbf{x}(t))]_j
\;\geq\;
\tau
\max_{j\in\mathcal{P}_i}[\nabla F(\mathbf{x}(t))]_j,
~
\forall i\in[N],\;\forall t\in[0,1].
\end{equation}
Let \(\mathbf{x}^\star\in\arg\max_{\mathbf{x}\in\mathcal{M}}F(\mathbf{x})\).
Then
\[
F(\mathbf{x}(t))
\;\geq\;
\bigl(1-e^{-\tau t}\bigr)F(\mathbf{x}^\star),
\qquad
\forall t\in[0,1].
\]
In particular,
\[
F(\mathbf{x}(1))
\;\geq\;
\bigl(1-e^{-\tau}\bigr)F(\mathbf{x}^\star).
\]
\end{thm}
The proof is given in the appendix.

\begin{remark}[Comparison with classical CG]
\label{rem:cg_comparison}
Theorem~\ref{thm:atcg_tau} shows that \textit{ATCG} preserves the 
same exponential improvement structure as classical continuous greedy, 
but with rate parameter $\tau$ in place of $1$. In particular, 
classical CG satisfies 
$F(\mathbf{x}(1)) \geq (1-e^{-1})F(\mathbf{x}^\star)$,
whereas \textit{ATCG} satisfies
$F(\mathbf{x}(1)) \geq (1-e^{-\tau})F(\mathbf{x}^\star)$.
Thus, the thresholding mechanism effectively replaces the exact CG 
oracle with a $\tau$-approximate oracle induced by the active sets 
$\{\mathcal{A}_i\}$. When $\tau=1$, the bound recovers the classical 
guarantee. As $\tau$ decreases, the approximation factor degrades 
smoothly according to the active-set coverage quality, while the 
monotone ascent structure is preserved throughout. The numerical 
results further illustrate that $\tau$ is especially beneficial in 
server-assisted settings, where communication cost scales with 
$\sum_i|\mathcal{A}_i|$. By triggering expansion only when 
$\eta_i < \tau$~\eqref{eq:eta_def} and keeping only the most 
informative candidates active, \textit{ATCG} significantly reduces 
communication overhead while preserving strong optimization 
performance.\boxend
\end{remark}

Theorem~\ref{thm:atcg_tau} establishes a worst-case performance 
guarantee for \textit{ATCG} under the $\tau$-coverage 
condition~\eqref{eq:tau-coverage}. The resulting rate $1-e^{-\tau}$ 
is a conservative baseline, determined entirely by the threshold 
parameter $\tau$ and independent of any structural properties of the 
submodular objective. In practice, however, when $f$ has low total 
curvature, the performance of \textit{ATCG} can be substantially 
stronger.

 The total curvature $c\in[0,1]$ of a submodular function 
is defined as~\cite{MC-GC:84}
\begin{equation}\label{eq:curvature}
c = 1 - \min_{\mathcal{S}\subset\mathcal{P},\,
p\notin\mathcal{S}}
\frac{f(\mathcal{S}\cup\{p\})-f(\mathcal{S})}
{f(\{p\})},
\end{equation}
and measures how far $f$ deviates from modularity: $c=0$ 
corresponds to an additive function, for which an optimal 
solution is recoverable in polynomial time, while $c=1$ 
captures the strongest diminishing returns. For 
$0 < c < 1$, \cite{MC-GC:84} established that the 
sequential greedy algorithm achieves an improved 
approximation ratio of $\tfrac{1}{c}(1-e^{-c})$, which 
tightens beyond $(1-1/e)$ as $c$ decreases from $1$ 
toward $0$. This motivates examining how curvature 
interacts with the threshold parameter in \textit{ATCG}.

Intuitively, low curvature means that the marginal 
contribution of an element does not deteriorate 
significantly as other elements become 
active~\cite{vondrak1978submodularity}. In the context of \textit{ATCG}, this is especially 
favorable: if partition $\mathcal{P}_i$ already contains a highly 
informative candidate in its active set $\mathcal{A}_i$, low 
curvature implies that this candidate remains competitive throughout 
the trajectory, even as $\mathbf{x}$ evolves. As a result, the 
progress ratio $\eta_i$~\eqref{eq:eta_def} can remain well above the 
nominal threshold $\tau$, and the effective oracle quality of 
\textit{ATCG} approaches that of the full CG method.

The next result formalizes this observation. It shows that if, in 
each partition, the active set contains the best singleton element, then 
$\eta_i$ is automatically lower-bounded by $1-c$, where $c$ is the 
total curvature of $f$. Consequently, the effective convergence rate 
improves from $\tau$ to $\max\{\tau, 1-c\}$.

\begin{thm}[Curvature-aware guarantee for \textit{ATCG}, Algorithm~\ref{alg:atcg}]
\label{cor:atcg_curvature}
Under the assumptions of Theorem~\ref{thm:atcg_tau}, let $f$ have 
total curvature $c\in[0,1]$. For each partition $\mathcal{P}_i$, define
\[
j_i^\circ \in \operatorname*{arg\,max}_{j\in\mathcal{P}_i}
f(\{j\}),
\]
and suppose that $
j_i^\circ \in \mathcal{A}_i(t),
\qquad
\forall\,i\in [N]=\{1,\ldots,N\},\;\forall\,t\in[0,1].$ 
Then, for every partition $i$ and every $t\in[0,1]$,
\[
\max_{j\in\mathcal{A}_i(t)}[\nabla F(\mathbf{x}(t))]_j
\;\geq\;
(1-c)\max_{j\in\mathcal{P}_i}[\nabla F(\mathbf{x}(t))]_j.
\]
Consequently, the \textit{ATCG} trajectory satisfies
\[
F(\mathbf{x}(t))
\;\geq\;
\bigl(1-e^{-\max\{\tau,\,1-c\}\,t}\bigr)\,F(\mathbf{x}^\star),
\qquad \forall\,t\in[0,1].
\]
Then, at $t=1$, $
F(\mathbf{x}(1))
\;\geq\;
\bigl(1-e^{-\max\{\tau,\,1-c\}}\bigr)\,F(\mathbf{x}^\star).
$
\end{thm}
The proof is given in the appendix.

\begin{rem}[Special cases and Pareto structure]\label{rem:pareto}
When $c=0$, $\max\{\tau,1-c\}=1$ and the bound recovers
the classical CG rate $1-1/e$, regardless of $\tau$.
When $\tau=1-c$, the rate is $1-e^{-(1-c)}$.
For $\tau\le 1-c$, the effective rate $1-c$ is independent
of $\tau$; any $\tau<1-c$ achieves identical performance,
making such choices dominated. Increasing $\tau$ above
$1-c$ improves the approximation toward $1-1/e$ as
$\tau\to 1$, but increases communication as shown in
Section~\ref{sec::com_eff}. Hence
$\tau^\star=1-c$ is the \textit{Pareto-optimal threshold},
simultaneously attaining minimum communication and the best
approximation factor at that cost.\hfill$\square$
\end{rem}


\subsection{Communication Efficiency}\label{sec::com_eff}

Let $C_{\mathrm{CG}}(T)$ and $C_{\mathrm{ATCG}}(T)$ denote the
cumulative number of feature embeddings uploaded to the server
through step $T$ under CG and \textit{ATCG}, respectively.
In \textit{ATCG}, a feature embedding is transmitted only when
its element first enters an active set, so communication is
governed by the growth of $\{\mathcal{A}_i\}_{i=1}^N$.
Potentially, standard \textit{CG} can select a new element
per partition at every iteration as $\mathbf{x}$ evolves, so
$C_{\mathrm{CG}}(T)\to|\mathcal{P}|$ as $T\to\infty$;
\textit{ATCG} prevents this through its per-partition threshold
mechanism: the expansion step in Algorithm~\ref{alg:atcg}
fires only when $\eta_i(t)<\tau$, directly gating which
embeddings are ever transmitted.

\begin{lem}[Communication Dominance]\label{lem:comm_dom}
Let $j_i^{\mathrm{CG}}(t)\in\arg\max_{j\in\mathcal{P}_i}
[\nabla F(\mathbf{x}(t))]_j$ denote the CG oracle selection
in partition~$i$ at iteration~$t$ along the \textit{ATCG}
trajectory, so that
$C_{\mathrm{CG}}(T) = \sum_{i=1}^N
\bigl|\bigcup_{t=0}^{T-1}\{j_i^{\mathrm{CG}}(t)\}\bigr|$.
Then $C_{\mathrm{ATCG}}(T)\le C_{\mathrm{CG}}(T)$
for all $T\ge 1$.
\end{lem}

\begin{proof}
Fix partition~$i$. Whenever \textit{ATCG} expands
$\mathcal{A}_i$ at iteration~$t$, $\eta_i(t)<\tau\le 1$,
so $\max_{j\in\mathcal{A}_i}[\nabla F]_j <
\max_{j\in\mathcal{P}_i}[\nabla F]_j$; hence
$j_i^{\mathrm{CG}}(t)\notin\mathcal{A}_i(t)$ and
\textit{ATCG} selects
\[
\arg\max_{j\in\mathcal{P}_i\setminus\mathcal{A}_i(t)}
[\nabla F(\mathbf{x}(t))]_j
= j_i^{\mathrm{CG}}(t),
\]
up to consistent tie-breaking.
Therefore
$\mathcal{A}_i(T)\subseteq
\bigcup_{t=0}^{T-1}\{j_i^{\mathrm{CG}}(t)\}$,
and summing over all partitions gives
$C_{\mathrm{ATCG}}(T)\le C_{\mathrm{CG}}(T)$.\boxend
\end{proof}

Since $\eta_i(t)\in[0,1]$ is computed from $K$ Monte Carlo
gradient estimates via~\eqref{eq::grad}, it fluctuates around
a nominal value $\bar{\eta}_i(t)$ whose variability decreases
as $K$ grows. To quantify the role of $\tau$ explicitly, we
adopt the following one-sided bound as a modeling assumption:
\begin{equation}\label{eq::threshold_prob}
\mathbb{P}\!\left(\eta_i(t)<\tau\right)
\;\le\;
\Phi\!\left(
\frac{\tau-\bar{\eta}_i(t)}{\sigma_i}
\right),
\quad \forall\,i,t,\tau,
\end{equation}
where $\Phi$ is the standard Gaussian CDF and $\sigma_i>0$
captures the variability of $\eta_i(t)$. This is a
conservative tail bound on the activation probability, not
a claim that $\eta_i(t)$ is Gaussian: when
$\tau\ll\bar{\eta}_i(t)$ violations are rare, and when $\tau$
increases the bound grows monotonically, reflecting the
stricter coverage~requirement. 

\begin{lem}[Expected Communication]\label{lem:exp_comm}
Under the threshold-violation model~\eqref{eq::threshold_prob}, let
\(T>1\). Then, 
{\small \[
\mathbb{E}\!\left[C_{\mathrm{ATCG}}(T)\right]
\le
\min\!\left\{
\mathbb{E}\!\left[C_{\mathrm{CG}}(T)\right],
N
\!+\!\sum_{t=1}^{T-1}\sum_{i=1}^N
\Phi\!\left(
\frac{\tau-\bar{\eta}_i(t)}{\sigma_i}
\right)\!
\right\}.
\]}
\end{lem}

\begin{proof}
Define $Z_{i,t}=\mathbf{1}\{\eta_i(t)<\tau,\,
\mathcal{A}_i(t)\ne\mathcal{P}_i\}$.
Since \(\mathcal{A}_i(0)=\emptyset\), \textit{ATCG} activates
one largest-marginal element from each partition at the first iteration,
which accounts for \(N\) initial transmitted feature embeddings. Hence, for \(T>1\)
$C_{\mathrm{ATCG}}(T)=N+\sum_{t=1}^{T-1}\sum_{i=1}^N Z_{i,t}$.
Taking expectations and applying
\(
\mathbb{P}\!\left(\eta_i(t)<\tau,\;
\mathcal{A}_i(t)\ne\mathcal{P}_i\right)
\;\le\;
\mathbb{P}\!\left(\eta_i(t)<\tau\right)
\;\le\;
\Phi\!\left(\frac{\tau-\bar{\eta}_i(t)}{\sigma_i}\right)
\)
yields the second bound.
Taking expectations in Lemma~\ref{lem:comm_dom}
gives the first.
Combining both bounds via the minimum completes the proof.\boxend
\end{proof}
This bound highlights that communication is driven not by the full gradient dynamics, but by rare excursions where the active set fails to capture sufficient marginal gain, so when $\bar{\eta}_i(t)$ remains above $\tau$ most of the time, the communication cost stays close to its minimum.

Next, we characterize the minimum
achievable communication savings. For this, we define the \textit{partition curvature}
\begin{equation*}
    c_i = 1 - \min_{\mathcal{S}\subset\mathcal{P},\;
    p\in\mathcal{P}_i\setminus\mathcal{S}}
    \frac{f(\mathcal{S}\cup\{p\})-f(\mathcal{S})}{f(\{p\})}
\end{equation*}
which measures marginal-gain deterioration for elements of
partition~$i$; under low curvature, a highly informative element
in $\mathcal{A}_i$ remains competitive as $\mathbf{x}$ evolves,
as formalized below.

\begin{lem}[Communication Complexity]\label{lem:comm}
Let $C_{\mathrm{ATCG}}(T)=\sum_{i=1}^{N}|\mathcal{A}_i(T)|$.
For any $i\in[N]$, if $\tau\le 1-c_i$, then
$\eta_i(t)\ge\tau$ for all $t\ge 1$,
so $|\mathcal{A}_i(T)|=1$ for all $T\ge 1$.
Moreover, if $\tau\le 1-\max_{i\in[N]}\{c_i\}$, then
$C_{\mathrm{ATCG}}(T)=N$ for all $T\ge 1$,
attaining the minimum possible communication of one
embedding per~partition.
\end{lem}

\begin{proof}
At $t=0$, $\mathcal{A}_i=\emptyset$ so $\eta_i=0<\tau$,
and \textit{ATCG} activates
$j_i^\circ=\arg\max_{j\in\mathcal{P}_i}f(\{j\})$
in each partition.
We show $\eta_i(t)\ge 1-c_i\ge\tau$ for all $t\ge 1$,
so the expansion condition is never triggered again. Since $j_i^\circ\in\mathcal{A}_i(t)$ for all $t\ge 1$
by monotonicity of the active sets,
\begin{equation*}
    \max_{j\in\mathcal{A}_i(t)}[\nabla F(\mathbf{x}(t))]_j
    \;\ge\; [\nabla F(\mathbf{x}(t))]_{j_i^\circ}
    \;\ge\; (1-c_i)\,f(\{j_i^\circ\}),
\end{equation*}
where the second inequality uses the partition curvature
lower bound $[\nabla F(\mathbf{x})]_j\ge(1-c_i)f(\{j\})$
for $j\in\mathcal{P}_i$, obtained by taking expectations
in $f(\mathcal{S}\cup\{j\})-f(\mathcal{S})\ge(1-c_i)f(\{j\})$
over $\mathcal{S}=\mathcal{R}(\mathbf{x})$~\eqref{eq::grad}.
Submodularity (set $\mathcal{S}=\emptyset$) and expectations
via~\eqref{eq::grad} give
$[\nabla F(\mathbf{x}(t))]_j\le f(\{j\})\le f(\{j_i^\circ\})$
for all $j\in\mathcal{P}_i$:
\begin{equation*}
  \frac{\max_{j\in\mathcal{A}_i(t)}
           [\nabla F(\mathbf{x}(t))]_j}
          {\max_{j\in\mathcal{P}_i}
           [\nabla F(\mathbf{x}(t))]_j}
    \;\ge\;
    \frac{(1-c_i)\,f(\{j_i^\circ\})}{f(\{j_i^\circ\})}
    = 1-c_i.
\end{equation*}
Therefore,   $\eta_i(t)\geq 1-c_i\;\ge\;\tau.$
The expansion condition is never triggered after $t=0$,
so $|\mathcal{A}_i(T)|=1$ for all $T\ge 1$.
Applying this to every $i\in[N]$ under
$\tau\le 1-\max_i c_i$ gives $C_{\mathrm{ATCG}}(T)=N$.
\end{proof}

Combined with Theorem~\ref{cor:atcg_curvature},
Lemma~\ref{lem:comm} completes the Pareto picture
established in Remark~\ref{rem:pareto}:
since $c_i\le c$ for all $i$, the per-partition bound
gives $1-\max_i c_i\ge 1-c$, strictly widening the
safe regime relative to the global curvature bound,
and as $\max_i c_i\to 0$ this expands to $\tau\in(0,1]$,
permitting $C_{\mathrm{ATCG}}(T)=N$ and near-optimal
$(1-1/e)$ performance~simultaneously.


\begin{figure}[t]
\centering
\includegraphics[scale=0.3]{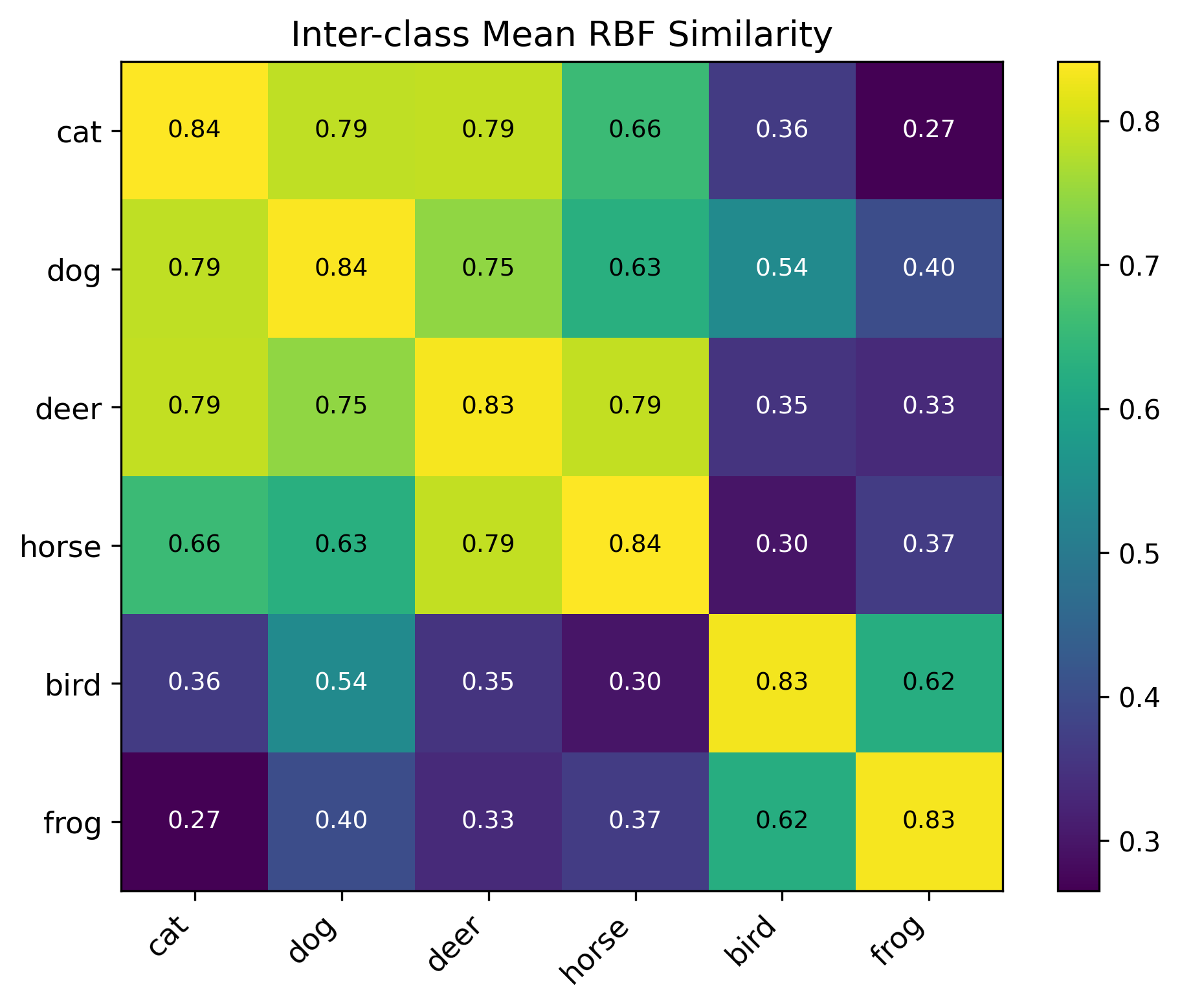}
\caption{\small Objective trajectory \(F(\mathbf{x}_t)\) versus iteration for CG, SCG,
and \textit{ATCG} with different threshold values \(\tau\) on the CIFAR
animal subset. }
\label{rbf:fig}
\end{figure}

\section{Numerical Evaluation}\label{numerical}
We evaluate \textit{ATCG} on two experiments spanning
different data modalities and objectives to demonstrate
the robustness of the communication--performance tradeoff
controlled by $\tau$. The first uses a CIFAR-10 animal
subset under a facility-location objective to illustrate
active-set stabilization and objective tracking relative
to CG and SCG. The second uses the MovieLens dataset
under a collaborative-filtering objective to show that
the role of $\tau$ is consistent and, in this case, more
pronounced across a wider range of operating points.

\subsection{CIFAR-10 Animal Subset: Active-Set Stabilization}
We consider a target monitoring task involving a group of $N$ robots. Each robot is provided with a large, labeled image archive $\mathcal{D}$ consisting of images from $N$ target categories (e.g., $N$ distinct groups of animals). These archival images can be visualized as the blue data points in Fig.~\ref{fig:forward_gradiant} within a feature space where similar categories naturally form clusters, though some features may overlap. Each robot $i$ independently collects live target images from its patrol area, producing a local image set $\mathcal{P}_i$ (like the cross points in Fig.~\ref{fig:forward_gradiant})\footnote{In Fig.~\ref{fig:forward_gradiant}, the matroid constraint is a partition matroid, with non-overlapping local ground sets defined for each agent $i \in [N] = \{1, \dots, N\}$ as $\mathcal{P}_i = \{(i, b) \mid b \in \mathcal{B}_i\}$, where $\mathcal{B}_i$ represents the selection choices of agent $i$.}. The union $\mathcal{P} = \bigcup_{i=1}^{N} \mathcal{P}_i$ constitutes the full ground set of candidate images distributed across the team. 
Using the labeled archive as a reference, the team performs \emph{exemplar clustering}~\cite{krause2014submodular} so that each agent selects a representative target point from $\mathcal{P}_i$. Collectively, these choices are intended to form a set of $N$ data points that best characterize each target category so each agent can focus on one representation. An agent's choice directly impacts others; if one agent selects a specific representative to monitor, the other robots must focus on covering representatives from the remaining classes to ensure diverse coverage.

We instantiate this scenario on a subset of CIFAR-10 comprising six animal categories (deer, frog, bird, horse, cat, and dog). We randomly sample a subset of data from this dataset to construct the local set $\mathcal{P}_i$ of each agent. Images are embedded via a pretrained ResNet, and pairwise similarities are computed with the RBF kernel $K(p,q) = \exp(-\|\mathbf{z}_p - \mathbf{z}_q\|_2^2 / (2\sigma^2))$. The global utility is defined by the facility-location objective $f(\mathcal{S}) = \sum_{p \in \mathcal{P}} \max_{q \in \mathcal{S}} K(p,q)$ for any $\mathcal{S}\subseteq\mathcal{P}$, which is monotone submodular and measures how effectively the selected prototypes collectively cover all observed images in $\mathcal{P}$. The optimal multi-agent selection problem requires that $\mathcal{S}\in\mathcal{I}$, where $\mathcal{I}$ is the partition matroid defined in~\eqref{eq::partion_matroid}. The inter-class RBF similarity matrix is shown in Fig.~\ref{rbf:fig} confirms non-negligible cross-partition similarity among visually related categories such as deer, horse, cat, and dog. This coupling renders the objective non-separable across partitions, necessitating the globally coordinated active-set expansion of \textit{ATCG}. Under the server-assisted protocol of Section~\ref{sec:algo_theory}, the server
facilitates local active feature exchange and sharing aggregated global probability membership vector $\vect{x}$ among agents. 

\begin{figure}[t]
\centering
\includegraphics[scale=0.3]{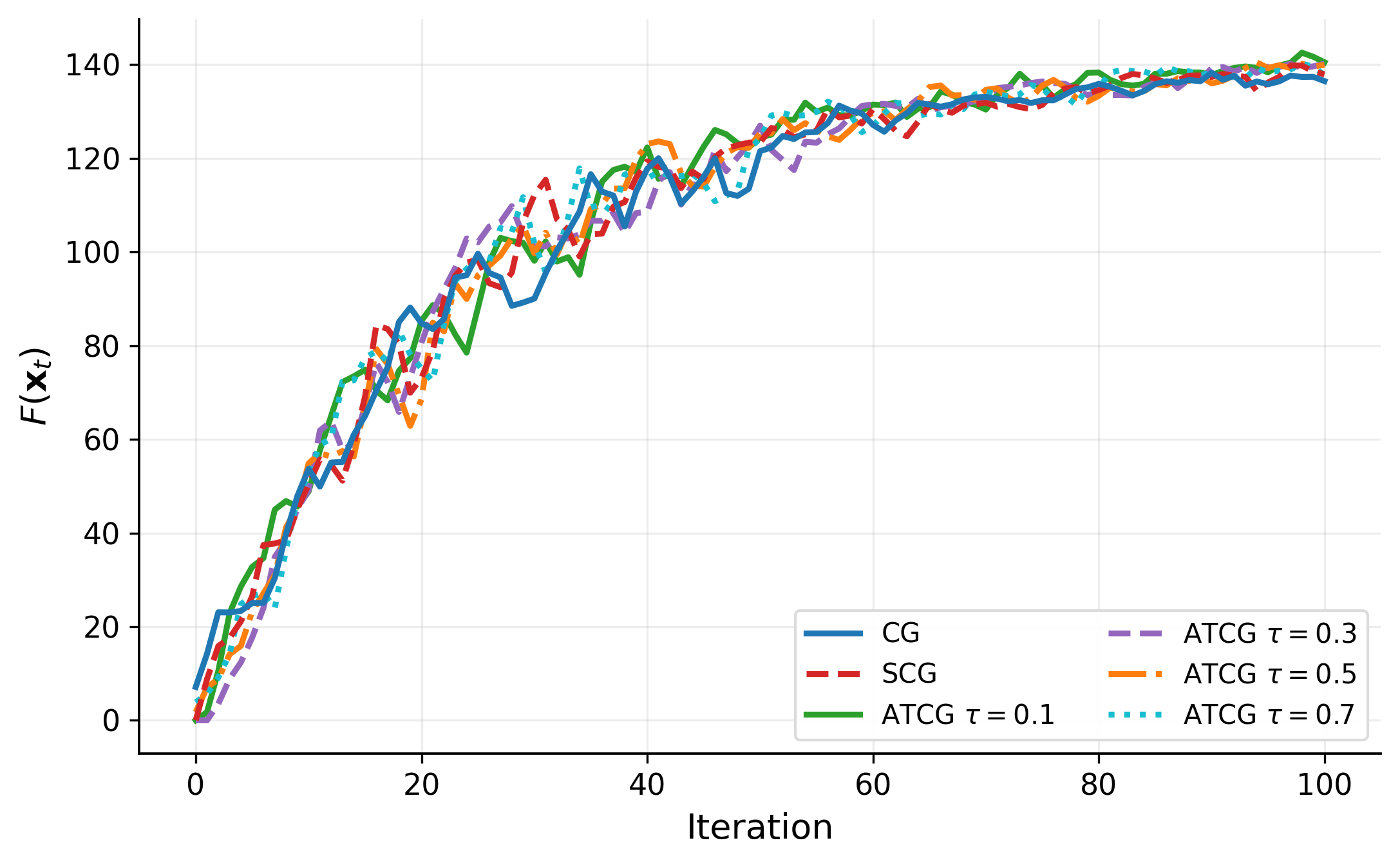}
\caption{\small Objective trajectory \(F(\mathbf{x}_t)\) versus iteration for CG, SCG,
and \textit{ATCG} with different threshold values \(\tau\) on the CIFAR
animal subset. }
\label{fig:traj_cifar_atcg}
\end{figure}

\begin{figure}[t]
\centering
\includegraphics[scale=0.3]{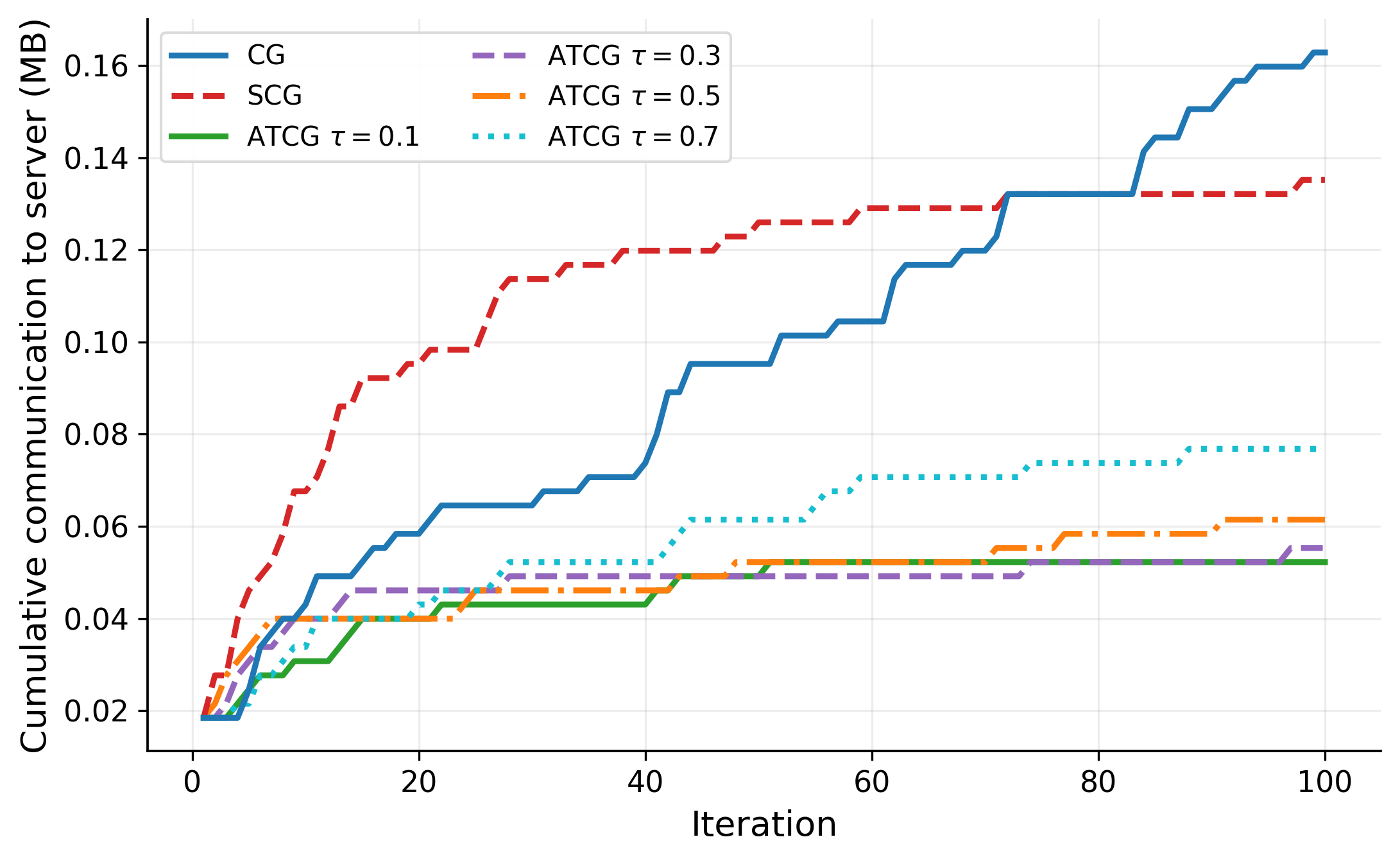}
\caption{\small Cumulative communication uploaded to the server for CG, SCG, and
\textit{ATCG} with different threshold values \(\tau\). 
}
\label{fig:comm_cifar_atcg}\vspace{-0.1in}
\end{figure}

\textbf{Objective quality (Fig.~\ref{fig:traj_cifar_atcg})} shows that
\textit{ATCG} tracks CG throughout all \(100\) iterations without visible
deviation across the tested threshold values. For \(\tau=0.3\), the final
discrete values after partition-wise argmax rounding are \(144.89\) (CG)
and \(143.63\) (\textit{ATCG})---a gap below \(1\%\)---confirming the
prediction of Theorem~III.1: restricting gradient evaluations to the
active sets \(\{\mathcal{A}_i\}\) via the \(\tau\)-coverage
rule~\eqref{eq:eta_def} introduces only a controlled approximation loss
relative to full~CG.


\textbf{Communication reduction (Fig.~\ref{fig:comm_cifar_atcg}).}
In contrast to CG and SCG \cite{mokhtari2018conditional}, whose communication grows steadily as $\mathbf{x}$
densifies, \textit{ATCG}'s cumulative upload curve bends over and becomes
completely flat well before the optimization terminates. This is the direct
signature of active-set stabilization, i.e., once every $\eta_i \geq \tau$, the
expansion rule in Algorithm~2 ceases to fire and no further feature
embeddings are sent to the server--yet the objective continues to improve on
the already-cached~information.

\begin{figure}[t]
\centering
\includegraphics[scale=0.3]{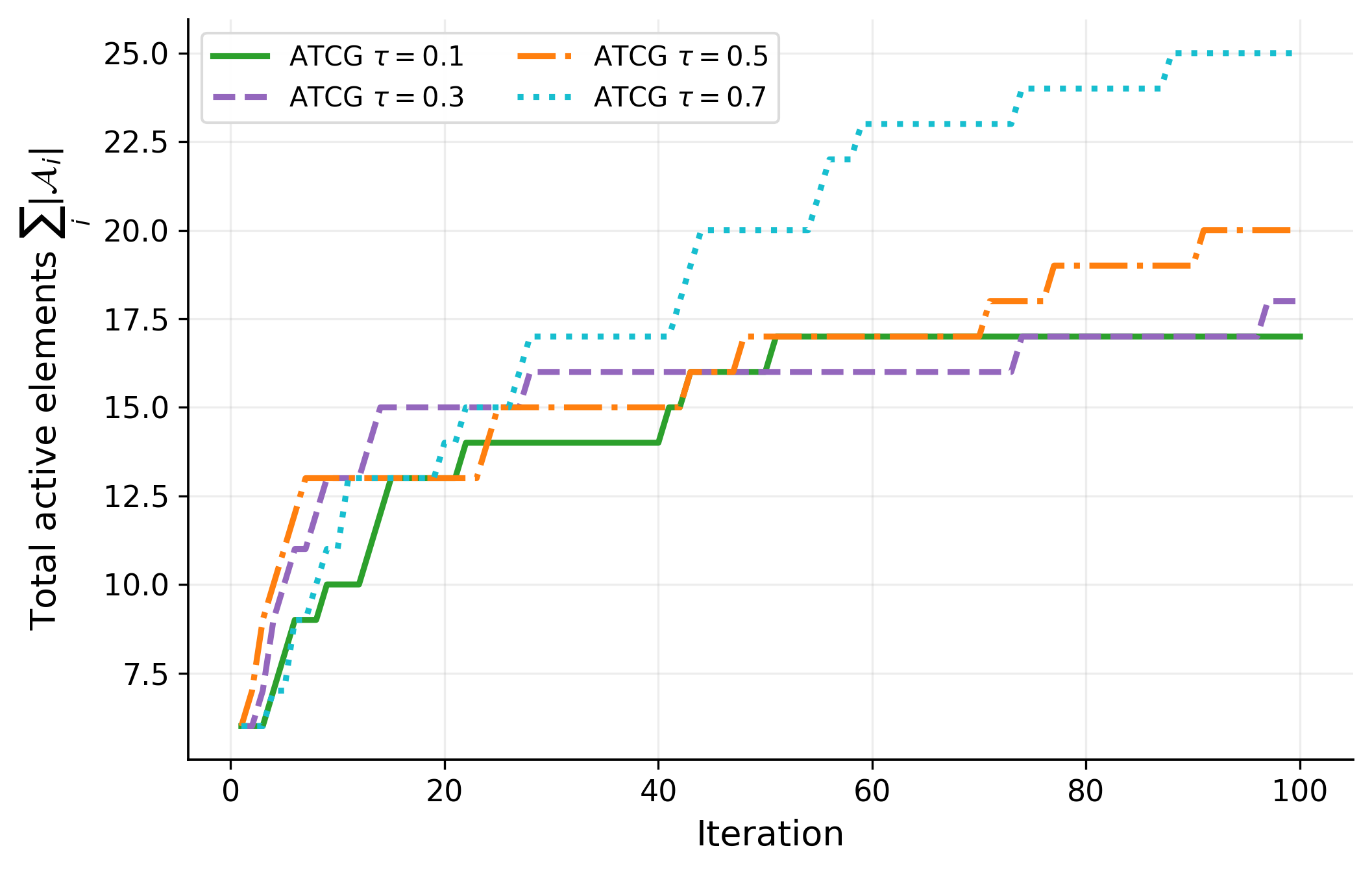}
\caption{\small Total active-set size \(\sum_i|\mathcal{A}_i|\) versus iteration for
\textit{ATCG} with different threshold values \(\tau\). 
}
\label{fig:active_growth_cifar_atcg}\vspace{-0.2in}
\end{figure}

\textbf{Active-set dynamics (Fig.~\ref{fig:active_growth_cifar_atcg}).}
Figure~\ref{fig:active_growth_cifar_atcg} makes the communication cutoff
mechanistically transparent. In the early iterations, several partitions
have \(\eta_i < \tau\), so the expansion step in \textit{ATCG} fires and
\(\sum_i|\mathcal{A}_i|\) grows. Once sufficiently informative candidates
are admitted, the partitions satisfy \(\eta_i \geq \tau\) and the curves
level off to constants. These plateaus directly correspond to the flat
regions of Fig.~\ref{fig:comm_cifar_atcg}: after stabilization,
optimization proceeds entirely on cached embeddings with zero additional
uploads. The plateau values \(\sum_i|\mathcal{A}_i| \ll |\mathcal{P}|\)
represent the total communication cost of \textit{ATCG}, consistent with
the communication bound established in Section~\ref{sec:algo_theory}.
Moreover, the different plateau values across \(\tau\) illustrate the
expected threshold-dependent tradeoff: larger \(\tau\) values impose a
stricter active-set coverage requirement and therefore activate more
elements, while still preserving objective values close to full~CG.

\subsection{MovieLens: Communication--Performance Tradeoff}

To further examine the communication--performance tradeoff of
\textit{ATCG}, we also evaluate the proposed method on the MovieLens
dataset. We use the facility-location objective
$
f(S)=\frac{1}{N}\sum_{u=1}^{N}\max_{j\in S} r_{u,j},
$
where \(r_{u,j}\) denotes the rating of user \(u\) for movie \(j\). Thus,
the selected set \(S\) is evaluated by the average user satisfaction,
where each user is represented by the highest-rated movie available in
the selected set. To match the partition-matroid setting, the movie
ground set is split into balanced partitions and \textit{ATCG} is
compared with CG and SCG.

The results in Fig.~\ref{fig:movielens_traj_atcg} show that the effect
of the threshold parameter \(\tau\) is more pronounced in this example.
CG achieves \(f(S)\approx 4.82\) with \(87\) communicated elements, while
SCG reaches a comparable value \(f(S)\approx 4.78\) but requires \(204\)
communicated elements. The larger communication cost of SCG is caused by
the stochasticity in its gradient estimates, which makes the oracle
select a more variable set of elements across iterations and therefore
activates more distinct elements. In contrast, \textit{ATCG} with
\(\tau=0.3\) communicates only \(16\) elements and obtains
\(f(S)\approx 4.58\), while increasing the threshold to \(\tau=0.7\)
improves the value to \(f(S)\approx 4.81\) using only \(33\)
communicated elements. Hence, larger \(\tau\) improves objective quality
at the cost of higher communication, while still requiring substantially
fewer communicated elements than both CG and SCG. This additional
experiment further confirms that \(\tau\) directly controls the
communication--performance tradeoff of \textit{ATCG}.

\begin{figure}[t]
    \centering
    \includegraphics[scale=0.31]{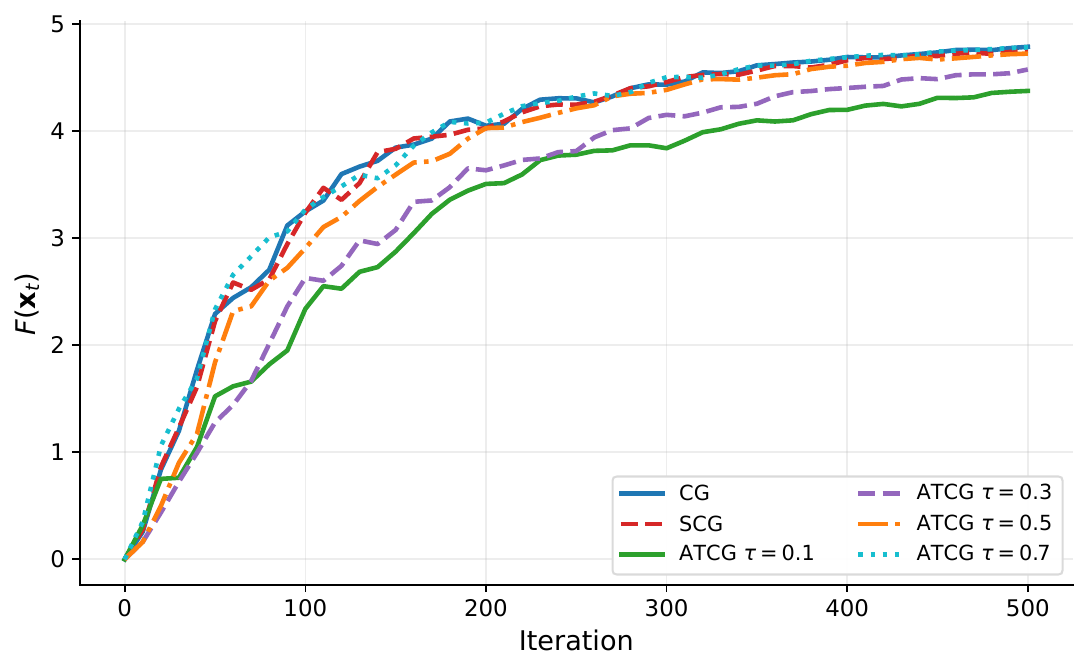}
    \caption{
    Objective trajectory \(F(\mathbf{x}_t)\) on the MovieLens facility-location
    experiment. The results compare CG, SCG, and \textit{ATCG} for different
    values of \(\tau\). Larger values of \(\tau\) improve the objective value
    and approach the CG/SCG baselines, while smaller values of \(\tau\) lead to
    lower communication at the cost of reduced objective value.
    }
    \label{fig:movielens_traj_atcg}
\end{figure}

\section{Conclusion}\label{sec::conclu}
This paper introduced \textit{ATCG}, a threshold-driven variant of the continuous greedy algorithm for submodular maximization under partition-matroid constraints. By restricting gradient evaluations to dynamically expanding active sets, \textit{ATCG} preserves the $(1-e^{-\tau})$ approximation of the classical method while substantially reducing communication overhead. A curvature-aware analysis shows the effective rate improves to $1-e^{-\max\{\tau,\,1-c\}}$, recovering the $(1-1/e)$ guarantee as curvature vanishes. Importantly, the established framework can facilitate the convergence analysis of hard-thresholding variants, where coordinates are explicitly zeroed. Future work will explore smart policies to proactively plan thresholding and activation rather than waiting for set saturation, alongside extensions to fully decentralized networks.

\bibliographystyle{ieeetr}%
\bibliography{bib/alias,bib/Reference}

\appendix

\renewcommand{\theequation}{A.\arabic{equation}}
\renewcommand{\thethm}{A.\arabic{thm}}
\renewcommand{\thelem}{A.\arabic{lem}}
\renewcommand{\thedefn}{A.\arabic{defn}}

\section*{A. Proof of convergence theorems}
This section presents the proofs of the results in the paper.

\begin{proof}[Proof of Theorem~\ref{thm:atcg_tau}]
Let $\mathbf{v}_{\mathrm{ATCG}}(t)$ denote the ascent direction 
selected by \textit{ATCG} at time $t$, and let
\[
\bar{\mathbf{v}}(t)\in \operatorname*{arg\,max}_{\mathbf{v}
\in\mathcal{M}}\;
\mathbf{v}^{\!\top}\nabla F(\mathbf{x}(t))
\]
denote the exact continuous greedy direction. Recall that this oracle decomposes partition-wise 
via~\eqref{eq::oracle_decomp}. For each partition $i$, the exact oracle 
selects
\[
j_i^\star(t)\in \operatorname*{arg\,max}_{j\in\mathcal{P}_i} 
[\nabla F(\mathbf{x}(t))]_j,
\]
while \textit{ATCG} selects
\[
\hat{j}_i(t)\in \operatorname*{arg\,max}_{j\in\mathcal{A}_i(t)} 
[\nabla F(\mathbf{x}(t))]_j.
\]
By the $\tau$-coverage condition~\eqref{eq:tau-coverage},
\[
[\nabla F(\mathbf{x}(t))]_{\hat{j}_i(t)}
\;\geq\;
\tau\,[\nabla F(\mathbf{x}(t))]_{j_i^\star(t)},
\qquad \forall\, i.
\]
Summing over all partitions yields
\[
\mathbf{v}_{\mathrm{ATCG}}^{\!\top}(t)\,\nabla F(\mathbf{x}(t))
\;\geq\;
\tau\;
\bar{\mathbf{v}}^{\!\top}(t)\,\nabla F(\mathbf{x}(t)).
\]
Since $\mathbf{x}^\star\in\mathcal{M}$, optimality of 
$\bar{\mathbf{v}}(t)$ gives
\[
\bar{\mathbf{v}}^{\!\top}(t)\,\nabla F(\mathbf{x}(t))
\;\geq\;
{\mathbf{x}^\star}^{\!\top}\nabla F(\mathbf{x}(t)).
\]
Moreover, a standard property of the multilinear extension of a 
monotone submodular function states that~\cite{calinescu2011maximizing}
\[
{\mathbf{x}^\star}^{\!\top}\nabla F(\mathbf{x}(t))
\;\geq\;
F(\mathbf{x}^\star)-F(\mathbf{x}(t)).
\]
Combining the three inequalities gives
\[
\mathbf{v}_{\mathrm{ATCG}}^{\!\top}(t)\,\nabla F(\mathbf{x}(t))
\;\geq\;
\tau\bigl(F(\mathbf{x}^\star)-F(\mathbf{x}(t))\bigr).
\]
Since $\dot{\mathbf{x}}(t)=\mathbf{v}_{\mathrm{ATCG}}(t)$, the 
chain rule gives
\[
\frac{d}{dt} F(\mathbf{x}(t))
=
\nabla F(\mathbf{x}(t))^{\!\top}\dot{\mathbf{x}}(t)
=
\nabla F(\mathbf{x}(t))^{\!\top}\mathbf{v}_{\mathrm{ATCG}}(t).
\]
Therefore,
\[
\frac{d}{dt} F(\mathbf{x}(t))
\;\geq\;
\tau\bigl(F(\mathbf{x}^\star)-F(\mathbf{x}(t))\bigr).
\]
Define $g(t):=F(\mathbf{x}^\star)-F(\mathbf{x}(t))\geq 0$. 
Then $g'(t)\leq -\tau g(t)$. By Gr\"{o}nwall's 
inequality~\cite{khalil2002nonlinear}, and using 
$F(\mathbf{x}(0))=F(\mathbf{0})=0$,
\[
g(t)\leq e^{-\tau t}g(0)=e^{-\tau t}F(\mathbf{x}^\star).
\]
Equivalently,
\[
F(\mathbf{x}(t))
\;\geq\;
\bigl(1-e^{-\tau t}\bigr)\,F(\mathbf{x}^\star).
\]
Setting $t=1$ completes the proof. Applying pipage or swap rounding 
to the fractional solution $\mathbf{x}(1)$ then yields a feasible 
discrete set $\mathcal{S}_{\textit{ATCG}}\in\mathcal{I}$ satisfying
\[
\mathbb{E}\bigl[f(\mathcal{S}_{\textit{ATCG}})\bigr]
\;\geq\;
\bigl(1-e^{-\tau}\bigr)\,f(\mathcal{S}^\star),
\qquad
\mathcal{S}^\star \in 
\operatorname*{arg\,max}_{\mathcal{S}\in\mathcal{I}} f(\mathcal{S}).
\]
\end{proof}

\medskip
\begin{proof}[Proof of Theorem~\ref{cor:atcg_curvature}]
Fix any $t\in[0,1]$ and any partition $i$. Since 
$j_i^\circ\in\mathcal{A}_i(t)$ by assumption,
\[
\max_{j\in\mathcal{A}_i(t)}[\nabla F(\mathbf{x}(t))]_j
\;\geq\;
[\nabla F(\mathbf{x}(t))]_{j_i^\circ}.
\]
By the curvature definition~\eqref{eq:curvature} for the multilinear extension, for every 
$j\in\mathcal{P}$~\cite{vondrak1978submodularity},
\[
(1-c)\,f(\{j\})
\;\leq\;
[\nabla F(\mathbf{x}(t))]_j
\;\leq\;
f(\{j\}).
\]
Applying the lower bound to $j_i^\circ$ gives
\[
[\nabla F(\mathbf{x}(t))]_{j_i^\circ}
\;\geq\;
(1-c)\,f(\{j_i^\circ\}).
\]
On the other hand, the upper bound applied to any $j\in\mathcal{P}_i$ 
yields
\[
\max_{j\in\mathcal{P}_i}[\nabla F(\mathbf{x}(t))]_j
\;\leq\;
\max_{j\in\mathcal{P}_i} f(\{j\})
=
f(\{j_i^\circ\}).
\]
Combining these three inequalities,
\[
\max_{j\in\mathcal{A}_i(t)}[\nabla F(\mathbf{x}(t))]_j
\;\geq\;
(1-c)\max_{j\in\mathcal{P}_i}[\nabla F(\mathbf{x}(t))]_j.
\]
Since $t$ and $i$ were arbitrary, this bound holds for all 
$t\in[0,1]$ and all partitions $i$. Thus, in addition to the 
threshold-based coverage level $\tau$ from 
Theorem~\ref{thm:atcg_tau}, the active sets also satisfy a 
curvature-induced coverage level of $1-c$. The effective oracle 
quality is therefore at least
\[
\tau_{\mathrm{eff}} = \max\{\tau,\,1-c\}.
\]
Repeating the proof of Theorem~\ref{thm:atcg_tau} with 
$\tau_{\mathrm{eff}}$ in place of $\tau$ gives
\[
F(\mathbf{x}(t))
\;\geq\;
\bigl(1-e^{-\tau_{\mathrm{eff}}\,t}\bigr)\,F(\mathbf{x}^\star).
\]
Substituting $\tau_{\mathrm{eff}}=\max\{\tau,\,1-c\}$ completes 
the proof.
\end{proof}

\renewcommand{\theequation}{B.\arabic{equation}}
\renewcommand{\thethm}{B.\arabic{thm}}
\renewcommand{\thelem}{B.\arabic{lem}}
\renewcommand{\thedefn}{B.\arabic{defn}}

\section*{B. Extension to General Partition Budgets     $\kappa_i\ge 1$}

Under the partition matroid with budget $\kappa_i\geq 1$, the
matroid polytope restricted to partition~$i$ is
\begin{equation}\label{eq::polytope_kappa}
\mathcal{M}_i
=
\Bigl\{
\vect{w}\in[0,1]^{|\mathcal{P}|}
\;\Big|\;
\sum_{p\in\mathcal{P}_i}[\vect{w}]_p\le\kappa_i,\;
[\vect{w}]_p=0\;\;\forall\,p\in\mathcal{P}\setminus\mathcal{P}_i
\Bigr\}.
\end{equation}
At each iteration, the partition-$i$ component of the CG
linear oracle solves
\begin{equation}\label{eq::oracle_kappa}
\vect{v}_i(t)
=
\operatorname*{arg\,max}_{\vect{w}\in\mathcal{M}_i}
\;\vect{w}^{\!\top}\nabla F(\vect{x}(t)).
\end{equation}
Since \(f\) is monotone submodular, every gradient component
satisfies \([\nabla F(\vect{x}(t))]_p\ge 0\)
for all \(p\in\mathcal{P}_i\)~\cite{calinescu2011maximizing}.
The non-negativity of all components implies that the
maximizer of~\eqref{eq::oracle_kappa} is attained by placing
unit mass on the \(\kappa_i\) elements with the largest
gradient values in \(\mathcal{P}_i\). Formally, one optimal solution is
\begin{equation}\label{eq::v_kappa}
\vect{v}_i(t)
=
\mathbf{1}_{\{p_1^\star,\ldots,p_{\kappa_i}^\star\}},
\end{equation}
where \(\{p_1^\star,\ldots,p_{\kappa_i}^\star\}\subset\mathcal{P}_i\)
consists of the \(\kappa_i\) elements with the largest entries
of \(\{[\nabla F(\vect{x}(t))]_p\}_{p\in\mathcal{P}_i}\), i.e.,
\begin{equation}\label{eq::p_star_kappa}
\{p_1^\star,\ldots,p_{\kappa_i}^\star\}
=
\operatorname{Top}_{\kappa_i}
\!\Bigl(\bigl\{[\nabla F(\vect{x}(t))]_p
\bigr\}_{p\in\mathcal{P}_i}\Bigr).
\end{equation}
Using
\[
\vect{v}_i(t)=
\mathbf{1}_{\{p_1^\star,\ldots,p_{\kappa_i}^\star\}}
\]
in the discrete update
\[
\vect{x}_i\!\left(t+\frac{1}{T}\right)
=
\vect{x}_i(t)+\frac{1}{T}\vect{v}_i(t),
\]
only the \(\kappa_i\) coordinates of \(\vect{x}_i(t)\)
corresponding to \(\{p_1^\star,\ldots,p_{\kappa_i}^\star\}\)
are modified, each increasing by \(1/T\).
When \(\kappa_i=1\), this recovers exactly the single-element
update rule of Algorithm~\ref{alg:atcg}.

\subsection*{B.1\quad Generalized Progress Ratio}

For $\kappa_i\ge 1$, the quality of the active set
$\mathcal{A}_i$ is measured by comparing the top-$\kappa_i$
oracle restricted to $\mathcal{A}_i$ with the unrestricted
top-$\kappa_i$ oracle over the full partition $\mathcal{P}_i$.
Define
\begin{align}
\mathcal{B}_i
&\in
\operatorname*{arg\,max}_{
\mathcal{B}\subseteq\mathcal{A}_i,\,
|\mathcal{B}|\le\kappa_i}
\sum\nolimits_{j\in\mathcal{B}}[\mathbf{g}]_j,
\label{eq::B_kappa}
\\[4pt]
\mathcal{C}_i
&\in
\operatorname*{arg\,max}_{
\mathcal{C}\subseteq\mathcal{P}_i,\,
|\mathcal{C}|\le\kappa_i}
\sum\nolimits_{j\in\mathcal{C}}[\mathbf{g}]_j,
\label{eq::C_kappa}
\end{align}
where $[\mathbf{g}]_j\approx[\nabla F(\mathbf{x})]_j$ are
the Monte Carlo gradient estimates.
Since both maximizations are over non-negative quantities,
$\mathcal{B}_i$ and $\mathcal{C}_i$ are simply the top-$\kappa_i$
elements of $[\mathbf{g}]$ within $\mathcal{A}_i$ and
$\mathcal{P}_i$, respectively.
The \emph{generalized progress ratio} is then
\begin{equation}\label{eq::eta_kappa}
\eta_i^{(\kappa)}
=
\frac{
\displaystyle\sum\nolimits_{j\in\mathcal{B}_i}[\mathbf{g}]_j
}{
\displaystyle\sum\nolimits_{j\in\mathcal{C}_i}[\mathbf{g}]_j
+10^{-12}
}.
\end{equation}
The condition $\eta_i^{(\kappa)}\ge\tau$ means that the
active-set top-$\kappa_i$ oracle captures at least a
$\tau$-fraction of the full-partition top-$\kappa_i$ oracle
value, generalizing the $\kappa_i=1$ progress ratio
in~\eqref{eq:eta_def}.

\emph{Active-Set Expansion Rule}: For $\kappa_i=1$, a single element suffices to restore the
coverage condition. For $\kappa_i>1$, one element may not
be sufficient. The expansion step therefore becomes an
\emph{inner loop}: while $\eta_i^{(\kappa)}<\tau$ and
$\mathcal{A}_i\ne\mathcal{P}_i$, the agent repeatedly adds
the inactive element with the largest gradient:
\begin{equation}\label{eq::expansion_kappa}
j_i^\star\in
\operatorname*{arg\,max}_{j\in\mathcal{P}_i\setminus\mathcal{A}_i}
[\mathbf{g}]_j,
\qquad
\mathcal{A}_i\leftarrow\mathcal{A}_i\cup\{j_i^\star\}.
\end{equation}
After each addition, $\eta_i^{(\kappa)}$ is recomputed
from the updated $\mathcal{A}_i$. The inner loop terminates
as soon as $\eta_i^{(\kappa)}\ge\tau$ or
$\mathcal{A}_i=\mathcal{P}_i$.
When $\kappa_i=1$, the inner loop executes at most once,
recovering Algorithm~\ref{alg:atcg} exactly.
Algorithm~\ref{alg:atcg:kappa} summarizes the resulting
server-assisted implementation for general $\kappa_i$.

\begin{algorithm}[t]
\caption{\textit{ATCG} for General $\kappa_i\ge 1$}
\label{alg:atcg:kappa}
\small
\begin{algorithmic}[1]
\Require Ground set $\mathcal{P}=\bigcup_{i=1}^N\mathcal{P}_i$;
         budgets $\{\kappa_i\}$; threshold $\tau>0$;
         horizon $T$; MC samples $K$
\State \textbf{[Server]} $\mathbf{x}\leftarrow\mathbf{0}$;
       $\mathcal{E}\leftarrow\emptyset$
\State \textbf{[Agent $i$, $\forall i$]}
       $\mathbf{x}_i\leftarrow\mathbf{0}$;
       $\mathcal{A}_i\leftarrow\emptyset$
\For{$t=0,1,\ldots,T-1$}
  \State \textbf{[Server$\to$Agent $i$]} Broadcast
         $\mathcal{E}$ and $\mathbf{x}$
  \For{$i=1,\ldots,N$} \Comment{parallel}
    \State Estimate $[\mathbf{g}]_j\approx
           [\nabla F(\mathbf{x})]_j$, $\forall j\in\mathcal{P}_i$,
           using $K$ MC samples over $\mathcal{E}$
    \State Compute $\eta_i^{(\kappa)}$ via
           \eqref{eq::eta_kappa}
    \While{$\eta_i^{(\kappa)}<\tau$ \textbf{and}
           $\mathcal{A}_i\ne\mathcal{P}_i$}
           \Comment{inner expansion loop}
      \State $j_i^\star\leftarrow
             \arg\max_{j\in\mathcal{P}_i\setminus\mathcal{A}_i}
             [\mathbf{g}]_j$;
             $\mathcal{A}_i\leftarrow\mathcal{A}_i\cup\{j_i^\star\}$
      \State Transmit embedding of $j_i^\star$ to server;
             $\mathcal{E}\leftarrow\mathcal{E}\cup\{j_i^\star\}$
      \State Recompute $\eta_i^{(\kappa)}$
             via~\eqref{eq::eta_kappa}
    \EndWhile
    \State $\{p_1^\star,\ldots,p_{\kappa_i}^\star\}\leftarrow
           \operatorname{Top}_{\kappa_i}
           (\{[\mathbf{g}]_j\}_{j\in\mathcal{A}_i})$
    \State $[\mathbf{x}_i]_{p_k^\star}\leftarrow
           [\mathbf{x}_i]_{p_k^\star}+\tfrac{1}{T}$,
           $k=1,\ldots,\kappa_i$
  \EndFor
  \State \textbf{[Agent $i\to$Server]} Transmit
         $\mathbf{x}_i(t+\tfrac{1}{T})$, $\forall i$
  \State \textbf{[Server]} Assemble
         $\mathbf{x}(t+\tfrac{1}{T})=
         [\mathbf{x}_1^\top,\ldots,\mathbf{x}_N^\top]^\top$
\EndFor
\State \textbf{[Server]} Round $\mathbf{x}(1)$ to feasible
       $S\in\mathcal{I}$
\State \Return $S$
\end{algorithmic}
\end{algorithm}

\emph{Performance Guarantee}: The $\tau$-approximate oracle condition generalizes directly.
If $\eta_i^{(\kappa)}(t)\ge\tau$ for every partition $i$ and
every $t\in[0,1]$, then
\begin{equation}\label{eq::tau_coverage_kappa}
\sum_{i=1}^N\sum_{j\in\mathcal{B}_i(t)}
[\nabla F(\mathbf{x}(t))]_j
\;\ge\;
\tau
\sum_{i=1}^N\sum_{j\in\mathcal{C}_i(t)}
[\nabla F(\mathbf{x}(t))]_j,
\end{equation}
which is equivalently expressed as the inner-product
condition
\begin{equation}
\mathbf{v}_{\mathrm{ATCG}}(t)^\top
\nabla F(\mathbf{x}(t))
\;\ge\;
\tau\,\mathbf{v}_{\mathrm{CG}}(t)^\top
\nabla F(\mathbf{x}(t)).
\end{equation}
This is precisely the $\tau$-approximate oracle condition
used in the proof of Theorem~\ref{cor:atcg_curvature},
so the same continuous-greedy analysis yields
\[
F(\mathbf{x}(t))
\;\ge\;
\bigl(1-e^{-\tau t}\bigr)\,F(\mathbf{x}^\star),
\qquad\forall\,t\in[0,1],
\]
and, under the curvature-aware condition, the effective rate
improves to $\max\{\tau,1-c\}$ as established in
Theorem~\ref{cor:atcg_curvature}.

\emph{Communication Cost}

The communication cost is still governed by the cumulative
active-set size $C_{\mathrm{ATCG}}(T)=\sum_{i=1}^N|\mathcal{A}_i(T)|$.
The minimum possible communication is now
$\sum_{i=1}^N\kappa_i$ rather than $N$: at least $\kappa_i$
elements must be activated in partition~$i$ before the
top-$\kappa_i$ oracle can be evaluated.
Under the low-curvature condition $\tau\le 1-\max_i c_i$,
the inner expansion loop terminates after activating the
$\kappa_i$ best singleton elements in each partition,
so $|\mathcal{A}_i(T)|=\kappa_i$ for all $T\ge 1$, giving
\[
C_{\mathrm{ATCG}}(T)
=
\sum_{i=1}^N\kappa_i
\;\ll\;|\mathcal{P}|.
\]
The $\kappa_i=1$ case recovers $C_{\mathrm{ATCG}}(T)=N$
from Lemma~\ref{lem:comm}.

\end{document}